\documentclass{article}

 \usepackage{arxiv}
 \usepackage{natbib}

\usepackage{algorithm}
\usepackage{wrapfig}
\usepackage{lipsum}
\usepackage{algorithmic}
\usepackage{subcaption}
\usepackage[utf8]{inputenc} 
\usepackage[T1]{fontenc}    
\usepackage{hyperref}       
\usepackage{url}            
\usepackage{booktabs}       
\usepackage{amsfonts, amsmath, amssymb, amsthm, mathtools}       
\usepackage{nicefrac}       
\usepackage{microtype}      
\usepackage{xcolor}         
\usepackage{graphicx}

\usepackage{bm}
\newcommand{\gril}{\textsc{Gril}}

\newcommand{\zzgril}{\textsc{Zz-Gril}}
\newcommand{\qzzmod}{quasi zigzag persistence module}
\newcommand{\RKG}{\mathsf{rk}}
\newcommand{\RR}{\mathbb{R}}

\newcommand{\vp}{\mathbf{p}}

\newcommand{\zz}{\mathbb{ZZ}}
\newcommand{\z}{\mathbb{Z}}

\newcommand{\ropr}{\RR^{\mathsf{op}} \times \RR}
\newcommand{\uu}{\mathbb{U}}
\newcommand{\uuu}{\mathbb{UU}}
\renewcommand{\varprojlim}{\mathsf{lim}}
\renewcommand{\varinjlim}{\mathsf{colim}}

\theoremstyle{plain}
\newtheorem{theorem}{Theorem}[section]
\newtheorem{proposition}[theorem]{Proposition}
\newtheorem{lemma}[theorem]{Lemma}

\theoremstyle{definition}
\newtheorem{definition}[theorem]{Definition}

\theoremstyle{remark}

\newcommand{\propositionofref}{}
\newtheorem*{zpropositionof}{\textbf{Proposition} \propositionofref}
\newenvironment{propositionof}[1]
 {\renewcommand{\propositionofref}{#1}\zpropositionof}
 {\zpropositionof}

 \newcommand{\theoremofref}{}
\newtheorem*{ztheoremof}{\textbf{Theorem} \theoremofref}
\newenvironment{theoremof}[1]
 {\renewcommand{\theoremofref}{#1}\ztheoremof}
 {\ztheoremof}

\title{Quasi Zigzag Persistence: A Topological Framework for Analyzing Time-Varying Data}

\author{%
  Tamal K. Dey \\
  Department of Computer Science\\
  Purdue University\\
  West Lafayette, IN \\
  \texttt{tamaldey@purdue.edu} \\
  \And
  Shreyas N. Samaga \\
  Department of Computer Science \\
  Purdue University \\
  West Lafayette, IN \\
  \texttt{ssamaga@purdue.edu} \\
}

\begin{document}

\maketitle

\begin{abstract}
  In this paper, we propose Quasi Zigzag Persistent Homology (QZPH) as a framework for analyzing time-varying data by integrating multiparameter persistence and zigzag persistence. To this end, we introduce a stable topological invariant that captures both static and dynamic features at different scales. We present an algorithm to compute this invariant efficiently. We show that it enhances the machine learning models when applied to tasks such as sleep-stage detection, demonstrating its effectiveness in capturing the evolving patterns in time-varying datasets.
\end{abstract}

\section{Introduction}

Time varying data analysis~\cite{tsa,tsa2} has been a fundamental challenge in the machine learning community, ranging from traditional time-series data to more complex structures such as sequences of graphs or point clouds. While traditional time-series have been handled effectively~\cite{mvtsgnn}, the complexity of modern applications necessitates novel methodologies. Spatiotemporal Graph Neural Networks~\cite{stgcn,oreshkin2021,kan2022,chu2023} and specialized architectures for point cloud sequences~\cite{pcseq,fan2021pstnet,huang2021spatio,rempe2020caspr} have emerged as powerful tools. However, most of these approaches utilize local geometric information, potentially missing crucial global patterns. This is particularly apparent where the overall shape carries significant meaning, such as in brain connectivity patterns. Topological methods offer a compelling solution by capturing these global, multi-scale structures. By augmenting models with topological information, we can leverage both local geometric patterns and global characteristics for improved performance.

Recently, Topological Data Analysis (TDA) has become a prominent field for leveraging such hidden information. Persistent Homology (PH), a cornerstone of TDA, provides a succinct method to extract \emph{multiscale} topological features. This has been transformative in enhancing machine learning models. Analyzing time-varying data requires extending classical PH in two directions. First, the temporal component necessitates an additional parameter, bringing multiparameter persistence homology (MPH)~\cite{BL23} into the picture. Second, standard PH computations use monotone filtrations which cannot accommodate the deletions required for time-varying data. This calls for zigzag persistent homology (ZPH)~\cite{carlsson2010zigzag}, which captures dynamic topological features in time-series data~\cite{munch23, gel21, zz_visnet,HL24,coskunuzer2024time,beltramo2022euler}. Effectively, we need a combination of MPH and ZPH, allowing standard PH filtration along one parameter and ZPH along another. This requires structuring the underlying partially-ordered set (poset) as a \emph{quasi-zigzag poset}, leading to what we term \emph{Quasi Zigzag Persistent Homology (QZPH)}, which introduces a new set of challenges.

While MPH captures rich information, it suffers from the lack of a complete invariant, motivating the search for other informative, though incomplete, invariants~\cite{Multipers_landscapes, Multipers_Kernel_Kerber, Carriere_Multipers_Images, diff_signed_barcodes_24, vect_signed_barcodes_23, graphcode, gril23, dgril24}. The application of ZPH to MPH requires zigzag generalizations of these methods.


In this paper, we find that one such method, \gril{}~\cite{gril23} , which computes a landscape function~\cite{bubenik2015perslandscapes} using generalized ranks of intervals~\cite{GenRankKim21}, adapts naturally to the QZPH framework. This adaptation leads to our key contribution, \zzgril{}, a new topological invariant that extends the \gril{} framework to capture multiscale topological information in time-evolving data. On the theoretical front, we prove the stability of \zzgril{} and show that the generalized rank over a specific type of subposet can be computed efficiently using an algorithm from~\cite{DKM24}. This allows us to devise a practical algorithm for computing \zzgril{}. We demonstrate its value by augmenting machine learning models for tasks like sleep-stage detection from ECG and action classification from multivariate time-series, showing improved performance (Section~\ref{sec:exp}). \footnote{Code is available at \url{https://github.com/TDA-Jyamiti/zzgril}}.

A recent paper~\cite{flammer24} considers a similar setup for visualizing data with \emph{spatiotemporal persistence landscapes}. While the setup is similar, our work differs significantly. Notably, \cite{flammer24} use simple rectangular intervals, allowing a direct modification of a result in~\cite{DKM24}. In contrast, we use more general subposets, for which we prove Theorem~\ref{thm:compute_zzgril}, enabling the computation of generalized ranks for zigzag modules and leading to a different overall algorithm. We also provide an efficient algorithm to construct a quasi-zigzag bi-filtration from raw time-varying data and develop a complete ML pipeline that we validate on real-world datasets not considered in~\cite{flammer24}.


\section{Overview}
\label{sec:overview}

We begin with essential topological concepts. A \emph{simplicial complex} $K$ is a collection of subsets (simplices) of a finite vertex set $V$, such that if a simplex $\sigma$ is in $K$, then all its subsets are also in $K$. A (non-zigzag) \emph{filtration} $\mathcal F$ is a nested sequence of simplicial complexes indexed by
integers:
$\mathcal{F}: K_0\hookrightarrow K_1\hookrightarrow\cdots\hookrightarrow K_n$. If some inclusions are reversed, we get a \emph{zigzag filtration}, e.g., $\mathcal{Z}: K_0\hookrightarrow K_1\hookleftarrow \cdots\hookrightarrow K_n$.



Extending the indexing to two parameters such as a grid of integers
$\mathbb{Z}^2$, we get a bi-filtration (possibly zigzag). We will be interested
in a bi-filtration where the filtration in one direction, say $y$-direction,
is non-zigzag and the filtration in the other, $x$-direction, is zigzag.
For a moving point cloud data (PCD) such filtrations arise naturally. For example, consider
the case in Figure~\ref{fig:ZZ-PCD}. The moving PCD with 3 points are
shown at the bottom row. Along the $y$-direction, we increase the threshold
$\delta>0$ so that any two points are joined by an edge if the distance
between them is no more than $\delta$. Then, along $y$-direction, for every
PCD, we get a non-zigzag filtration whereas along $x$-direction, for every
fixed $\delta$, we get a zigzag filtration. We call this mixing of zigzag
and non-zigzag a \emph{quasi zigzag bi-filtration}.


\begin{figure}
    \centering
    \includegraphics[scale=0.6]{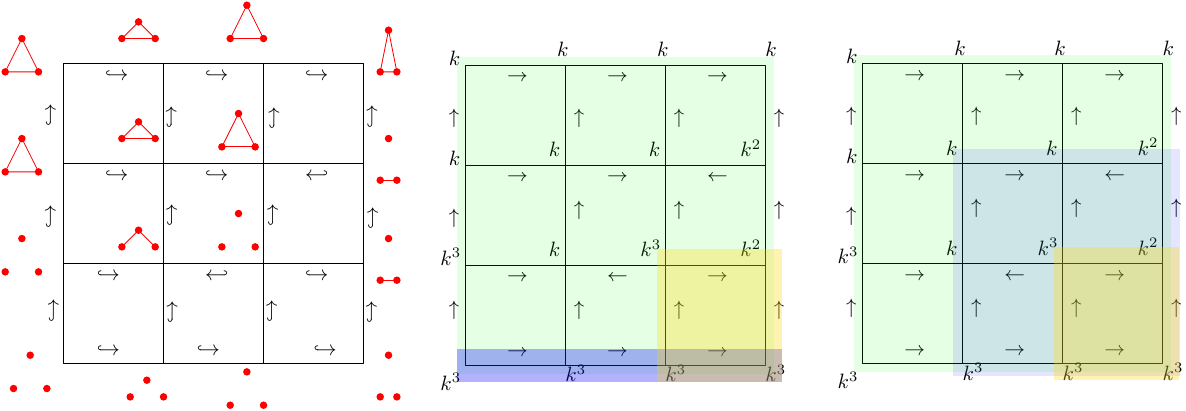}
    \caption{(left) A PCD (bottom row) and a resulting quasi zigzag bi-filtration; (middle) corresponding quasi zigzag persistence module and three intervals (light blue, yellow, green); (right) a rectangle (light yellow) is expanded
    for obtaining landscape width.}
    \label{fig:ZZ-PCD}
\end{figure}


The indexing set of a filtration is a partially ordered set (poset). A non-zigzag filtration can be indexed by the integers $\z$, while the quasi zigzag bi-filtration is indexed by a product poset $\zz \times \z$, where $\zz$ denotes a zigzag poset that arises in our case(see
Definition~\ref{def:zzposet}). Applying simplicial homology functor (e.g., $H_p$) to each complex in the filtration yields a \emph{persistence module}: a collection of vector spaces connected by linear maps according to the inclusions. For a quasi zigzag bi-filtration, this results in a \emph{quasi zigzg persistence module}. In Figure~\ref{fig:ZZ-PCD} (right), we show the
quasi zigzag persistence module for zeroth homology group $H_0(\cdot)$ that represents
the number of components.


\begin{definition}[Persistence module]
    A persistence module over a poset $P$ is a functor $M \colon P \to \mathbf{vec}$, i.e., a collection of vector spaces $\{M_\mathbf{p}\}_{\mathbf{p}\in P}$ along with linear maps $M_{\mathbf{p} \to \mathbf{q}} \colon M_{\mathbf{p}} \to M_{\mathbf{q}}$ for every $\mathbf{p} \leq \mathbf{q}$ in $P$. Here, $\mathbf{vec}$ denotes the category of finite dimensional vector spaces.
    \label{def:qzzpersmod}
\end{definition}

Next, we introduce the concept of \emph{intervals}. A finite poset $P$ is \emph{connected} if for every pair $p,q \in P$ there is a sequence of points $p=r_1,\ldots, r_t=q$ in $P$ so that $r_i \leq r_{i+1}$ or $r_{i+1}\leq r_i$ for all $i\in \{1,\ldots,t-1\}$. An \emph{interval} $I \subseteq P$ is a connected subposet that is convex w.r.t. poset order, i.e., if $p,q\in I$ and $p\leq r\leq q$, then $r \in I$.  Figure~\ref{fig:l_worm} depicts an interval in $\z^2$.


Our framework relies on the \emph{generalized rank}
~\cite{GenRankKim21} of a persistence module over a subposet $I \subseteq P$, which can be computed for $2$-parameter persistence modules with the algorithm~\cite{DKM24}. The generalized
rank of the module restricted to $I$ measures the multiplicity of the homological
features (\# of independent homological classes) that have support over the entire poset $I$. For example, in Figure~\ref{fig:ZZ-PCD}, the generalized rank, for $H_0$, over the entire poset (light green) is 1 as one component survives throughout, while the generalized rank over the bottom row (light blue) is 3 and that over the bottom right square (light yellow) is 2. 
\begin{definition}[Generalized rank]
    Let $M :P \to \mathbf{vec}$ be a persistence module, where $P$ is a finite connected poset. Let $M|_I$ denote the restriction of $M$ to a subposet $I$ of $P$. Then, the \emph{generalized rank} of $M$ over $I$ is defined as the rank of the canonical linear map from $\varprojlim M|_I$ to $\varinjlim M|_I$
    \begin{eqnarray*}
        \RKG^M(I) \coloneqq \text{rank} \left ( \varprojlim M|_I \to \varinjlim M|_I \right ).
    \end{eqnarray*}
\end{definition}
 We refer the reader to~\cite{Saunders_Maclane_Cat_Theory} for the definitions of limit, colimit, and the construction of the canonical limit-to-colimit map.

A key property of generalized rank is its \textbf{monotonicity}: $\RKG^M(I) \leq \RKG^M(J)$ for all $J \subseteq I$, where $I \text{ and } J$ are subposets in $P$. We use this to define \zzgril{}.

To extract the topological information from a \qzzmod{}, we cover the quasi zigzag poset with a specific type of subposets called \emph{worms}. A worm has a notion of a \emph{center} and a \emph{width}. Definition~\ref{def:worm} gives precise definitions of these terms. We expand each of these worms, i.e., increase the widths while keeping the centers fixed. The monotonicity of the generalized rank ensures that the rank can only decrease by this expansion. In Figure~\ref{fig:ZZ-PCD}(right) a subposet (rectangle) is expanded maximally to decrease the generalized rank from $3$ to $1$. The width for which the rank drops below a chosen threshold is taken as the \zzgril{} value. Defintion~\ref{def:zzgril} makes this concept precise. The set of \zzgril{} values at the chosen centers makes a vector that captures the topological information in the \qzzmod.

\section{\zzgril}\label{sec:zzgril}
In this section, we introduce concepts and definitions required for \zzgril. Then, we define \zzgril{} and discuss its theoretical properties. 

\begin{definition}[Zigzag poset]
    Zigzag poset $\zz$ is defined as a subposet of $\z^{\mathsf{op}} \times \z$ given by 
    \begin{equation*}
        \zz \coloneqq \{ (i,j) \colon i \in \z , j \in \{i, i-1\} \},
    \end{equation*}
    where $\z^{\mathsf{op}}$ denotes the opposite poset of $\z$.
    \label{def:zzposet}
\end{definition}


Consider the poset $\zz \times \z$ with the product order, i.e., $(\mathbf{z_1}, z_2) \leq (\mathbf{w_1}, w_2)$ if $\mathbf{z_1} \leq \mathbf{w_1}$ in $\zz$ and $z_2 \leq w_2$ in $\z$. Note that $\zz \times \z$ is equivalent to $\z^2$ as sets. Thus, every subposet $I\subseteq \zz\times \z$ can be thought as a subposet $I^{\z^2}\subseteq \z^2$ that
is endowed with $\z^2$ ordering. Observe that an interval in
$\z^2$ may not remain an interval in $\zz\times \z$ and vice-versa. 

We define a \emph{quasi zigzag bi-filtration} as a collection of simplicial complexes $\{K_{\mathbf{p}}\}_{\mathbf{p} \in \zz \times \z}$, where $K_{\mathbf{p}} \subseteq K_{\mathbf{q}}$ for all comparable $\mathbf{p} \leq \mathbf{q}$. Let $\mathbf{vec}$ denote
the category of finite dimensional vector spaces.

\begin{definition}[Quasi zigzag persistence module]\label{def:qzz_mod}
    A persistence module $M \colon \zz \times \z \to \mathbf{vec}$ is called a \emph{\qzzmod}. 
\end{definition}



\begin{figure}
    \centering
    \includegraphics[scale=0.6]{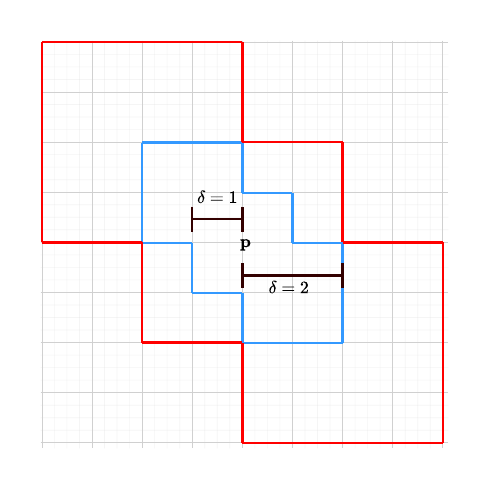}
    \vspace{-0.5cm}
    \caption{The worm with blue boundary represents a worm centered at $\vp$ with width $\delta=1$. The worm with red boundary represents an expanded worm, centered at $\vp$ with width $\delta=2$.}
    \label{fig:l_worm}
\end{figure}

We now define a special type of subposet in $\zz\times \z$ which in $\z^2$ is an $\ell$-worm introduced by~\cite{gril23} for $\ell=2$.

\begin{definition}[Worm]
    Let $\vp \in \zz \times \z$ and $\delta \in \z$ be given. Let $\boxed{\vp}_\delta$ denote the $\delta$-square centered at $\vp$, i.e., $\boxed{\vp}_\delta \coloneqq \{\mathbf{z} \in \zz\times \z :||\vp - \mathbf{z}||_\infty \leq \delta\}$. Then, a \emph{worm} centered at $\vp$ with \emph{width $\delta$} is defined as the union of the two $\delta$-squares $\boxed{\mathbf{q}}_\delta$ centered at points $\mathbf{q} = \mathbf{p}\pm(\delta, -\delta)$ on the off-diagonal line segment along with $\boxed{\vp}_\delta$. We denote the worm as $\boxed{\vp}_\delta^2$. The superscript denotes the number of $\delta$-squares in the union apart from $\boxed{\vp}_\delta$.
    \label{def:worm}
\end{definition}

We note that computing $|| \cdot ||_\infty$ in $\zz \times \z$ is equivalent to computing it in $\z^2$ using the notion of set equivalence. We can see that a worm turns out to be an interval in $\z^2$.
Refer to Figure~\ref{fig:l_worm} for an illustration of a worm.


\begin{definition}[\zzgril]
    Let $M$ be a \qzzmod. Then, the \emph{ZigZag Generalized Rank Invariant Landscape} is a function $\lambda^M :\zz \times \z \times \mathbb{N} \to \mathbb{N}$ defined as
    \begin{eqnarray*}
        \lambda^M(\mathbf{p}, k)\coloneqq \sup \left\{\delta \geq 0 :\RKG^M \left (\boxed{\mathbf{p}}^2_\delta \right)\geq k\right\}, 
    \end{eqnarray*}
    where $\vp \in \zz \times \z$.
    \label{def:zzgril}
\end{definition}

The basic idea of \zzgril{} is similar to \gril. However, we note that the underlying poset structure is very different. For computations, we consider a finite subposet of $\zz \times \z$ and cover it with worms. Then, we compute generalized rank over these worms to define the landscape function (\zzgril).

\subsection{Stability of \zzgril}
We prove the stability of \zzgril{} by showing that its perturbation is bounded by the interleaving distance between two \qzzmod s. The definition of interleaving distance, $d_{\mathcal{I}}(M,N)$, (see Definition \ref{def:interleaving} in Appendix~\ref{app:proofs}) between two zigzag persistence modules $M,N$ can be extended to \qzzmod s. 

There is an alternate notion of proximity on the space of persistence modules which uses Generalized Ranks computed on all intervals. This is known as \emph{erosion distance}~\cite{GenRankKim21}. We define a distance similar to erosion distance on the space of \qzzmod s based on their generalized ranks computed over worms. For this, we need the notion of $\epsilon$-\emph{thickening}.

Let \textbf{I}$(\zz \times \z)$ denote the collection of all subposets in $\zz \times \z$ such that their corresponding subposets in $\z^2$ are intervals. For $\epsilon \in \z_+$, the \emph{$\epsilon$-thickening} of $I$ is defined as 
\begin{equation*}
    I^\epsilon \coloneqq \left \{ \mathbf{r} \in \zz \times \z :\exists\mathbf{q} \in I \text{ such that } ||\mathbf{r} - \mathbf{q}||_\infty \leq \epsilon \right \}.
\end{equation*}

\begin{definition}
    Let $\mathcal{L}$ denote the collection of all worms in $\zz \times \z$. Let $M$ and $N$ be \qzzmod s. The erosion distance is defined as:
    \begin{equation*}
    \begin{split}
        d_\mathcal{E}^\mathcal{L}(M,N) \coloneqq \inf \limits_{\epsilon \geq 0} 
        \{ & \forall \boxed{\vp}^2_\delta \in \textbf{I}(\zz \times \z), \\
        & \RKG^M\left (\boxed{\vp}^2_\delta\right) \geq \RKG^N\left(\boxed{\vp}^2_{\delta+\epsilon}\right ) \text{ and } \\
        & \RKG^N\left (\boxed{\vp}^2_\delta\right ) \geq \RKG^M\left (\boxed{\vp}^2_{\delta+\epsilon}\right)  \}.
    \end{split}
    \end{equation*}
\end{definition}


Note that $\boxed{\vp}^2_{\delta+\epsilon}$ contains the $\epsilon$-thickening of $\boxed{\vp}^2_{\delta}$. 

\begin{proposition}\label{prop:interleaving}
    Given two \qzzmod s $M$ and $N$, $d_\mathcal{E}^{\mathcal{L}}(M,N) \leq d_\mathcal{I}(M,N)$ where $d_\mathcal{I}$ denotes the interleaving distance between $M$ and $N$.
\end{proposition}
\vspace{-0.2cm}


 Proposition~\ref{prop:interleaving} leads us to Theorem~\ref{thm:stability_zzgril}.

    

\begin{theorem}\label{thm:stability_zzgril}
    Let $M$ and $N$ be two \qzzmod s. Let $\lambda^M$ and $\lambda^N$ denote the \zzgril{} functions of $M$ and $N$  respectively. Then, 
    \begin{equation*}
        ||\lambda^M - \lambda^N||_\infty = d_\mathcal{E}^{\mathcal{L}}(M,N) \leq d_\mathcal{I}(M,N).
    \end{equation*}
\end{theorem}

    


We refer the reader to Appendix~\ref{app:proofs} for all the proofs.


\section{Algorithm}
\label{sec:alg}
We discuss the details of computing \zzgril{} in this section. We begin by proving that generalized rank over a finite subposet in $\zz \times \z$ can also be computed by computing the number of full bars in the zigzag filtration along the ``boundary'' of the  subposet. To be precise, we introduce the following concepts now. Let $P$ be any finite connected poset. We say a point $q\in P$ \emph{covers} a point $p\in P$, denoted
$p\prec q$ or $q\succ p$, if $p\leq q$ and there is no $r\in P$, $r\not= p,q$, so that $p\leq r\leq q$. A \emph{covering} path $p\sim q$ between $p$
and $q$ is either a set of points $\{p=p_0\prec p_1\prec\cdots\prec p_\ell=q\}$ 
or $\{p=p_0\succ p_1\succ\cdots\succ p_\ell=q\}$.
We define some special points in $P$. The set of minima in $P$
are the points $P_{min}=\{p\in P\ \,|\, \not\exists r\not=p\in P \mbox{ where } r\leq p\}$. Similarly, define the set of maxima $P_{max}$. For any two points
$p,q\in P$, let $p\vee q$ and $p\wedge q$ denote their least upper bound (lub)
and largest lower bound (llb) in $P$ if they exist. 

Now consider a finite subposet $I\subseteq \zz\times \z$ whose
counterpart $I^{\z^2}$ with $\z^2$ ordering becomes an interval.
Since $I^{\z^2}\subset \z^2$, its points
can be sorted with increasing $x$-coordinates in $\z^2$. 
Let $I^{\z^2}_{min}=\{p_0,\ldots, p_s\}$
and $I^{\z^2}_{max}=\{q_0,\ldots, q_t\}$ be the set of minima and maxima respectively sorted
according to their $x$-coordinates. The \emph{lower fence} and \emph{upper fence}
of $I^{\z^2}$ are defined as the paths
\vspace{-0.1cm}
\begin{eqnarray*}
L_I^{\z^2}:=p_0\sim (p_0\vee p_1)\sim p_1\sim (p_1\vee p_2)\sim p_2\sim\cdots\sim p_s\\
U_I^{\z^2}:=q_0\sim (q_0\wedge q_1)\sim q_1\sim(q_1\wedge q_2)\sim q_2\sim\cdots\sim q_t.
\end{eqnarray*}
Consider the boundary $B_I^{\z^2}$ comprising $L_I^{\z^2}$, $U_I^{\z^2}$, the paths
$p_0\sim q_0$ and $p_s\sim q_t$ going through the top left and bottom right corner points
of $I^{\z^2}$ respectively. Consider the subposet $I\subseteq \zz\times\z$ and
observe that $I_{min}, I_{max}\subseteq B_I^{\z^2}$ because all other points
have one point below and another above in the $y$-direction.
Let $\partial_L I$, $\partial_U I$, and $\partial I$ be minimal
paths in $BI^{\z^2}$ connecting all minima in $I_{min}$, all maxima in $I_{max}$, and
all minima, maxima together respectively. Drawing an analogy to \cite{DKM24} we call
$\partial I$ a \emph{boundary cap} which is drawn orange in Figure~\ref{fig:worm_qzz_bifil}.

Next, we appeal to certain results in category theory to claim a result
analogous to~\cite{DKM24} that helps computing the generalized ranks with
zigzag persistence modules. 

\begin{definition}
    A connected subposet $P'\subseteq P$ is called \emph{initial} if for every $p\in P\setminus P'$,
    the downset $\downarrow{p}=\{q\in P\,| \, q\neq p~\&~ q\leq p\}$ intersects $P'$
    in a connected poset. Similalrly, $P'$ is called \emph{terminal} if for every $p\in P\setminus P'$, the upset $\uparrow {p}=\{q\in P\,| \, q\neq p~\&~ q\geq p\}$ intersects $P'$
    in a connected poset.
\end{definition}

\begin{definition}
    \label{adef:init_functors}
    We call a functor $F: P'\rightarrow P$ for $P'\subseteq P$ \emph{initial} if $P'$ is initial.
Similarly, we call $F: P\rightarrow P'$ \emph{terminal} if $P'$ is terminal.
\end{definition}

Treating $I$ as a category with points in $I$ as objects and relations as morphisms, we get $\partial_LU$ and $\partial_UI$ as subcategories. Then, we get functors $F_L \colon \partial_LI \to I$ induced by inclusion and $F_U \colon I \to \partial_UI $ induced by projection.

\begin{proposition}
\label{aprop:init_functor}
    $F_L: \partial_L I\rightarrow I$  is an initial functor and $F_U: I\rightarrow \partial_U I$ is a terminal functor.
\end{proposition}
\begin{proof}
    It is sufficient to prove that $\partial_L I$ is initial for $F_L$ and $\partial_U I$ is terminal for $F_U$. We only show that $\partial_L I$ is initial and the proof
    for $\partial_U I$ being terminal is similar.

    Let $p\in I\setminus \partial_L I$ be any point. Let $p^-$ and $p^+$ be two points (if exist) where $p^-\prec p$ and $p^+ \succ p$ in $\z^2$ with $p^-_x < p_x < p^+_x$. By definition of the quasi zigzag poset, either $p^- > p$ and $p^+ > p$, or $p^- < p$ and
    $p^+ < p$ in the poset $\zz\times \z$.

    In the first
    case, the downset $\downarrow p$ consists of all points $p'\neq p$ that are vertically below $p$, that is, 
    $p'_x=p_x$ and $p'_y< p_y$. This means $\downarrow p$ intersects $\partial_L I$
    in a connected poset $p'_0 <\cdots < p'_k$ where each $p'_i$ has the same $x$-coordinate as $p$.

    In the second case, the downset $\downarrow p$ consists of three sets of points $Y^{p^-}$, $Y^p$, and $Y^{p^+}$  
    that are vertically below $p^-$, $p$, and $p^+$ respectively. Each of $Y^{p^-}$,
    $Y^p$ and $Y^{p^+}$ intersects $\partial_L I$ in a connected poset. They can be
    disconnected as a whole only if there is a point $q\in \partial_L I$ 
    so that
    $p^-_x < q_x < p_x$ or $p_x < q_x < p^+_x$. But, neither is possible because
    $p^-$ and $p^+$ are points in $I$ covered by $p$.    
\end{proof}


The following result connecting initial (terminal) functor to the limit (colimit)
is well known, [Chapter 8]~\cite{Riehl14}. 

\begin{proposition}
    Let $M: I\rightarrow {\bf vec}$ be a persistence module and $F_L: \partial_L I\rightarrow I$
    and $F_U: I\rightarrow \partial_U I$ be initial and terminal functors. Then, there
    is an isomorphism $\phi: \varprojlim{M}\rightarrow \varprojlim{M|_{\partial_L I}}$ from the limit of $M$ to the limit of the restricted module $M|_{\partial_L I}$. Similarly, we have an isomorphism for colimits, $\psi: \varinjlim{M|_{\partial_U I}}\rightarrow \varinjlim{M}$.
    \label{prop:fences}
\end{proposition}

\begin{proposition}
    $\partial (\partial_L I)=\partial_L I$ and $\partial (\partial_U I)=\partial_U I$.
    \label{prop:boudnarycap}
\end{proposition}
\begin{proof}
    The proof is immediate from the definitions.
\end{proof}
For computations, we bank on the following result which
extends Theorem 3.1 in~\cite{DKM24} from $2$-parameter persistence modules to the quasi zigzag persistence modules.

\begin{theorem}\label{thm:compute_zzgril}
    Let $M$ be a quasi zigzag persistence module and
    $I$ be a finite interval in the corresponding quasi zigzag poset. Then, 
    $\RKG^M(I)=\RKG^{M}(\partial I)$.
\end{theorem}
\begin{proof}
   Let $r$ denote the map $r: \varprojlim{I}\rightarrow \varinjlim{I}$, and
   $\phi: \varprojlim{I}\rightarrow \varprojlim{\partial_L I}$ and
   $\psi: \varinjlim{\partial_U I}\rightarrow \varinjlim{I}$ be the isomorphisms guaranteed by Proposition~\ref{prop:fences}. Observe that $\RKG^M(I)= \mathrm{rank}(\psi^{-1}\circ r \circ \phi^{-1})$. Now let $r'$ denote the map
   $r': \varprojlim{\partial I}\rightarrow \varinjlim{\partial I}$. Consider the isomoprhisms 
   $\phi': \varprojlim{\partial I}\rightarrow \varprojlim{\partial (\partial_L I)}$ and
   $\psi': \varinjlim{\partial(\partial_U I)}\rightarrow \varinjlim{\partial I}$ again guaranteed by Proposition~\ref{prop:fences}. Then,
   $\RKG^M(\partial I)=\mathrm{rank}(\psi'^{-1}\circ r'\circ \phi'^{-1})$. But,
   $\psi^{-1}\circ r \circ \phi^{-1}=\psi'^{-1}\circ r'\circ \phi'^{-1}$ because of
   Proposition~\ref{prop:boudnarycap}.
\end{proof}
With this background, the pipeline for computing \zzgril{} in the QZPH framework
can be described as follows.
Suppose that we are given sequential data (sequence of point clouds, sequence of graphs, multivariate time series) with vertex-level correspondences between consecutive time steps. First, we build a quasi zigzag bi-filtration out of this data where time
increases in $x$-direction and the threshold for constructing complexes increases
in $y$-direction. The algorithm for building this bi-filtration out of raw data
is described in section~\ref{sec:building_qzz_bifil}
where we make certain choices to make it efficient.
Each simplicial complex in the bi-filtration is indexed by a finite grid $G$ in $\z^2$ because $\zz \times \z$ is equivalent to $\z^2$ as sets. 
We sample a set of center points $S$ from $G$ and compute \zzgril{} (Definition~\ref{def:zzgril}) at each of these center points. Taking advantage of Theorem~\ref{thm:compute_zzgril}, we compute \zzgril{} by computing
the zigzag filtration along the boundary cap (a path) of a worm centering each point
in $S$ and then computing the number of full bars in the corresponding zigzag persistence module obtained by applying the homology functor.



A sequence of point clouds (Figure~\ref{fig:ZZ-PCD}) or graphs (Figure~\ref{fig:worm_qzz_bifil}) or multivariate time-series data (refer section~\ref{sec:exp} for converting multivariate time-series to a sequence of point clouds or graphs) gives us a collection of simplicial complexes. Every simplex in each simplicial complex is assigned a weight which is derived from the input. We build a quasi zigzag bi-filtration from this collection of simplicial complexes.


\subsection{Building Quasi Zigzag Bi-filtration}
\label{sec:building_qzz_bifil}
We explain the algorithm by considering an example of sequential graph data, as shown in Figure~\ref{fig:seq_of_graphs}. We assume that each graph has edge-weights. In the top row of Figure~\ref{fig:seq_of_graphs}, notice that the first two graphs can not be linked by an inclusion in any direction, i.e., neither is a subgraph of the other. This is because, there is an inclusion of an edge $(v_2,v_3)$ as well as a deletion $(v_1,v_4)$. We circumvent this problem by clubbing all inclusions together followed by all deletions. This is equivalent to considering the union of the two graphs as the intermediary step and then deleting the edges which are not present in the previous graph. Refer to the bottom row of Figure~\ref{fig:seq_of_graphs}, where the intermediary union graphs are shown along with the original graphs.

A sequence of $T$ number of graphs is converted into a zigzag filtration $\mathcal{Z}_{L}$ of length $2T-1$ by the above procedure because every consecutive pair of graphs introduces a union in between.
For each graph in the zigzag filtration thus obtained, we construct its \emph{graph filtration} based on edge-weights; see e.g.~\cite{dey_wang_2022_book, edelsbrunner2010computational}. We ensure that the length of each graph filtration remains the same by considering the sublevel sets of exactly $L$ levels.
The zigzag filtration at the highest level, $\mathcal{Z}_{L}$, can be pulled back to each lower level $l$. This ensures that we get a zigzag filtration $\mathcal{Z}_l$ at each level $1 \leq l < L$ of the filtration. This gives us a quasi zigzag bi-filtration. Refer to Figure~\ref{fig:worm_qzz_bifil} for an illustation of a quasi zigzag bi-filtration.
 
While implementing this procedure, we need not compute the union graphs explicitly saving both storage and time because all necessary information is already present in the corresponding component graphs. We use the following procedure to efficiently build the quasi zigzag bi-filtration.

Construct the graph filtration $\mathcal{F}_{t_i}$ of each of the $T$ graphs in the sequential graph data, where $1\leq t_i \leq T$. Let $\sigma$ be a simplex that is inserted at the level $l$ in the filtration $\mathcal{F}_{t_i}$. 
We have three possible scenarios for $\sigma$ in $\mathcal{F}_{t_{i+1}}$:
    \begin{enumerate}
    \item $\sigma$ is inserted in $\mathcal{F}_{t_{i+1}}$ at $m$ for some $m < l$: In this case, $\sigma$ needs to be added on all the horizontal inclusion arrows $K_{t_i, w} \xhookrightarrow{} K_{t_i, w} \cup K_{t_{i+1}, w}$ for $m \leqslant w < l$.
    
    \item $\sigma$ is inserted in $\mathcal{F}_{t_{i+1}}$ at $m$ for some $m \geqslant l$: In this case, $\sigma$ needs to be added on all the horizontal inclusion arrows $K_{t_{i+1}, w} \xhookrightarrow{} K_{t_i, w} \cup K_{t_{i+1}, w}$ for $l \leqslant w < m$.
    
    \item $\sigma$ is not present in $\mathcal{F}_{t_{i+1}}$. In this case, we treat it as getting inserted at the maximum level $L$ in $\mathcal{F}_{t_{i+1}}$ and hence, will be added by the inclusion $K_{t_{i+1}, L} \xhookrightarrow{} K_{t_i, L} \cup K_{t_{i+1}, L}$.
\end{enumerate}

We repeat this procedure for the pair $\mathcal{F}_{t_i}$ and $\mathcal{F}_{t_{i-1}}$.
\begin{figure}
    \centering
    \includegraphics[scale=0.7]{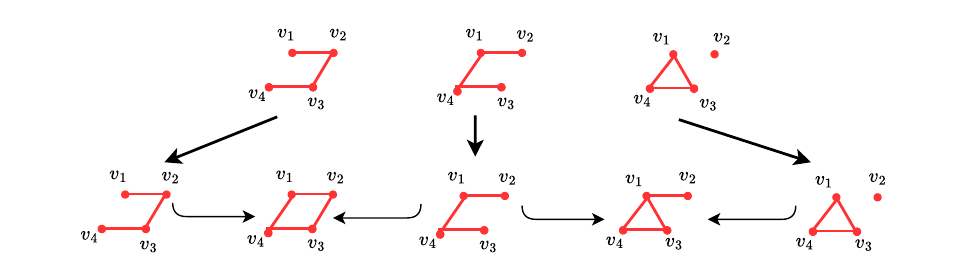}
    \caption{The top row shows a typical input to the \zzgril{} framework as a sequence of graphs. This sequence has 3 graphs. Observe that consecutive graphs cannot be linked by an inclusion in either direction because neither is a subgraph of the other. In order to circumvent this problem, we consider the unions of consecutive graphs as an intermediary step, as shown in the bottom row.}
    \label{fig:seq_of_graphs}
\end{figure}

Since we are working with an integer grid, we can simulate the quasi zigzag bi-filtration such that we have the standard filtration $\mathcal{F}_{t_i}$ at even $x$-coordinates and at odd $x$-coordinates, we have the standard filtration of the unions, without explicitly storing the unions.

\textbf{Correctness.} To prove the correctness of the procedure, we need to show 
the following: For $1\leq i,i' \leq T$ and $1\leq j,j' \leq L$, let $K_{i,j}$ denote the simplicial complex in the quasi zigzag filtration built according to the procedure mentioned above and $K'_{i,j}$ is the simplicial complex in the quasi zigzag bi-filtration built by considering unions explicitly. We show
that the new simplices inserted by the inclusion $K_{i,j} \xhookrightarrow{} K_{i',j'}$ are the same as the ones inserted by $K'_{i,j} \xhookrightarrow{} K'_{i',j'}$. 
\begin{proof}
    The statement is immediately true for the inclusions $i'= i, j' = j+1$. This is because the new simplices in the upward ($y$-direction) inclusions are the ones born at the level $j+1$ in the input family of filtrations. Now, we prove the statement for the case where $i' = i \pm 1, j' = j$, i.e., the horizontal ($x$-direction) inclusions/deletions. Note that we have inclusions from $K_{i,j} \xhookrightarrow{} K_{i',j'}$ only when $i$ is an even integer and $i' = i \pm 1$. We prove it for the case $i' = i + 1, j' = j$. The proof for the other case is similar. Let $\sigma$ be a simplex added on the inclusion $K_{i,j} \xhookrightarrow{} K_{i+1,j}$. According to the procedure described above, it means that $\sigma$ is born at some level $j_1 \leq j$ in the filtration $\mathcal{F}_{i+2}$ and is not born before level $j$ in the filtration $\mathcal{F}_i$. Hence, $\sigma$ will be newly added on the arrow $K'_{i,j} \xhookrightarrow{} K'_{i+1,j}$ because $\sigma \in K'_{i+2,j}$ and $\sigma \notin K'_{i,j}$. This is because $K_{i,j} = K'_{i,j}$ for all even integers $i$, as they are the simplicial complexes part of the input filtrations $\mathcal{F}_i$. By the same arguments, all the newly added simplices on inclusions $K'_{i,j} \xhookrightarrow{} K'_{i+1,j}$ will be newly added simplices on $K_{i,j} \xhookrightarrow{} K_{i+1,j}$.
\end{proof}


\subsection{Computing \zzgril}

\begin{figure}
    \centering
    \includegraphics[scale=0.6]{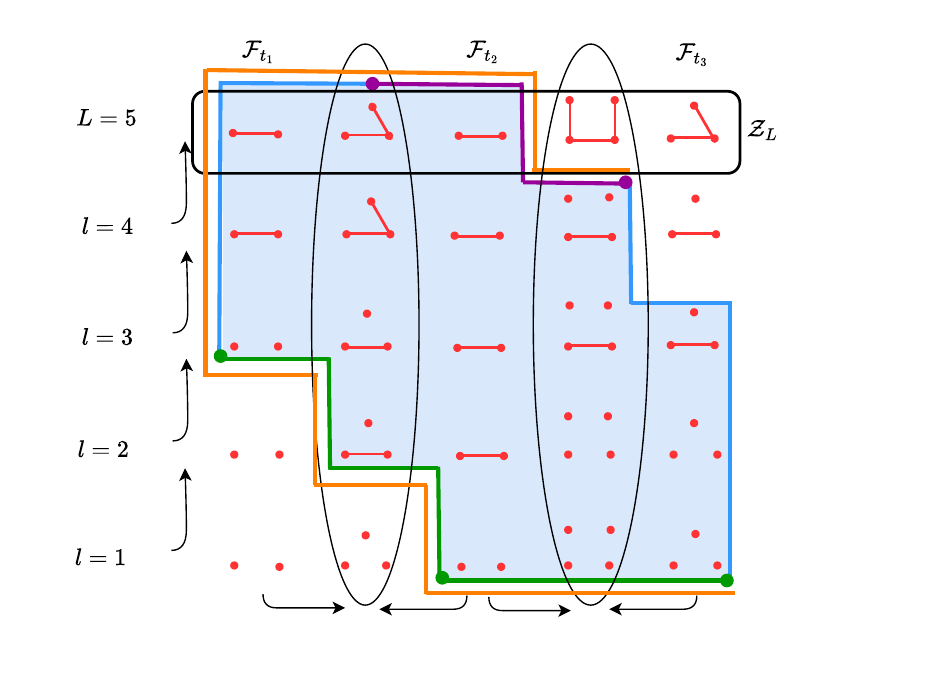}
    \caption{A worm $I$ (shaded) on a quasi zigzag bi-filtration where the direction of the inclusion on horizontal arrows are shown at the bottom and on the vertical arrows on left. The worm centered at (3,3) has width 1. The part of the boundary colored green is $\partial_LI$ connecting the minima shown as green points and the part in purple is $\partial_UI$ connecting the maxima shown as purple points. The \emph{boundary cap} is shown in orange which runs parallel to the boundary for a portion of it. Refer to Figure~\ref{fig:zz_bdry} for the zigzag filtration along the boundary cap of the worm. A sequence of graphs with $T=3$ time steps is shown with the corresponding graph filtrations: $\mathcal{F}_{t_1}, \mathcal{F}_{t_2}, \mathcal{F}_{t_3}$ each with $L=5$ levels. The filtration of the unions are  encircled by ovals. Zigzag filtration at the topmost level $\mathcal{Z}_L$ is shown in a rectangular box.}
    
    \label{fig:worm_qzz_bifil}
\end{figure}

\begin{figure*}
    \centering
    \includegraphics[width=\textwidth]{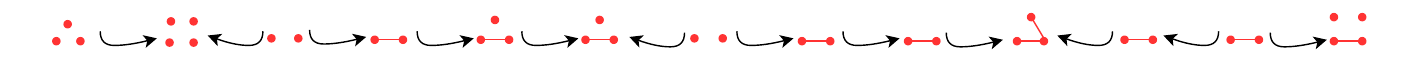}
    \caption{This is the zigzag filtration $\mathcal{Z}_{bdry}$ along the boundary cap of the worm shown in Figure~\ref{fig:worm_qzz_bifil}.}
    \label{fig:zz_bdry}
\end{figure*}
Given this quasi zigzag bi-filtration, we obtain a \qzzmod{} $M$ by considering the homology $\{H_p(K_{t_i, \alpha_j})\}_{i,j}$ of each simplicial complex. Given a sampled center point $\vp$ and the value of rank $k$, we need to compute $\lambda^M(\vp, k)$. Thus, we need to compute the maximal width of the worm such that the value of generalized rank of $M$ over the worm is at least $k$. We do a binary search over all the possible widths of the worms to arrive at such maximal width for each tuple $(\vp, k)$. To compute the generalized rank of $M$ over a worm, we need the zigzag filtration along the boundary cap of the worm, $\mathcal{Z}_{bdry}$. We utilize the stored information about which simplices to add/delete along the horizontal and vertical arrows in order to build $\mathcal{Z}_{bdry}$. Then, we compute the number of full bars in the zigzag persistence module corresponding to $\mathcal{Z}_{bdry}$ using an efficient zigzag algorithm proposed in~\cite{fzz} and use it in the binary search routine. The pseudocode of this idea is shown in Algorithm~\ref{alg:CompZZGRIL}.

\begin{algorithm}[h!]
\caption{\textsc{Compute\zzgril}}\label{alg:CompZZGRIL}
\begin{algorithmic}
\small 
   \STATE {\bfseries Input:} $\mathcal{ZZ}: \text{Quasi zigzag bi-filtration}, k\geq 1, \vp$
   \STATE {\bfseries Output:} $\lambda(\vp, k)$:  \zzgril{} value at point $\vp$ for  fixed $k$
   \STATE {\bfseries Initialize:} $\delta_{min} \gets 1, \delta_{max} \gets \text{len}(\mathcal{F})$, $\lambda \gets 1$ 
   \WHILE{$\delta_{min} \leq \delta_{max}$}
   \STATE $\delta \gets  (\delta_{min} + \delta_{max}) / 2 $
   \STATE $I \gets \boxed{\vp}^2_\delta$;~~ $r \gets$ \textsc{ComputeRank($\mathcal{ZZ}, I$)}
   
   \IF {$r \geq  k$}
        \STATE $\lambda \gets \delta$;~~ $\delta_{min} \gets \delta+1$
    \ELSE
        \STATE $\delta_{max} \gets \delta-1$ 
    \ENDIF
    \ENDWHILE
    
\STATE \textbf{return} {$\lambda$}

\end{algorithmic}
\end{algorithm}   

\subsection{Time Complexity Analysis}
\label{asubsec:complexity}
Let $N$ denote the length of the zigzag filtration at the top-most scale level, i.e., ($\mathcal{Z}_L$). For the analysis, we consider the zigzag filtration to be starting from an empty complex and ending at an empty complex. Assuming we have temporal data for $T$ time steps and that $L < T$ where $L$ denotes the number of different spatial scales, each iteration in the binary search requires $O(N^\omega)$ time since each iteration of the binary search involves computing zigzag persistence of a zigzag filtration of length $O(N)$. Here $\omega < 2.37286$ is the matrix multiplication exponent. Hence, the binary search in total requires $O(N^\omega \log T)$ time per center point. If we sample $s$ center points, then the total time complexity would be $O(sN^\omega \log T)$.

\section{Experiments}
\label{sec:exp}
In this section, we report the results of testing \zzgril{} on various datasets. We begin by giving a detailed description about the experimental setup. Then, we give a brief description of the datasets, followed by the experimental results. We use the benchmark UEA multivariate time-series~\cite{ueamvts} datasets to test \zzgril{} on multivariate time series data to show that \zzgril{} can be applied to datasets from various domains. Further, we test \zzgril{} on a targeted application of sleep-stage classification by performing experiments on ISRUC-S3~\cite{isruc} dataset. In all these experiments, we augment the topological information captured by \zzgril{} to one of the specifically tailored machine learning methods on the respective datasets and compare. 
For each case, we select the machine learning model which has the highest performance to truly test and highlight the value of the topological information added by \zzgril{} to an already high-performing specifically tailored model. 

\vspace{-0.2cm}
\subsection{Experimental Setup}
\vspace{-0.2cm}
 Each data instance is a multivariate time-series which we convert into a quasi zigzag bi-filtration. The \zzgril{} framework takes in a quasi zigzag bi-filtration as input and provides a topological signature for the sequence as output. We augment the machine learning model with this topological information and train the model for classification. The framework is shown in Figure~\ref{fig:pipeline}.

\begin{figure}
    \centering
    \includegraphics[scale=0.55]{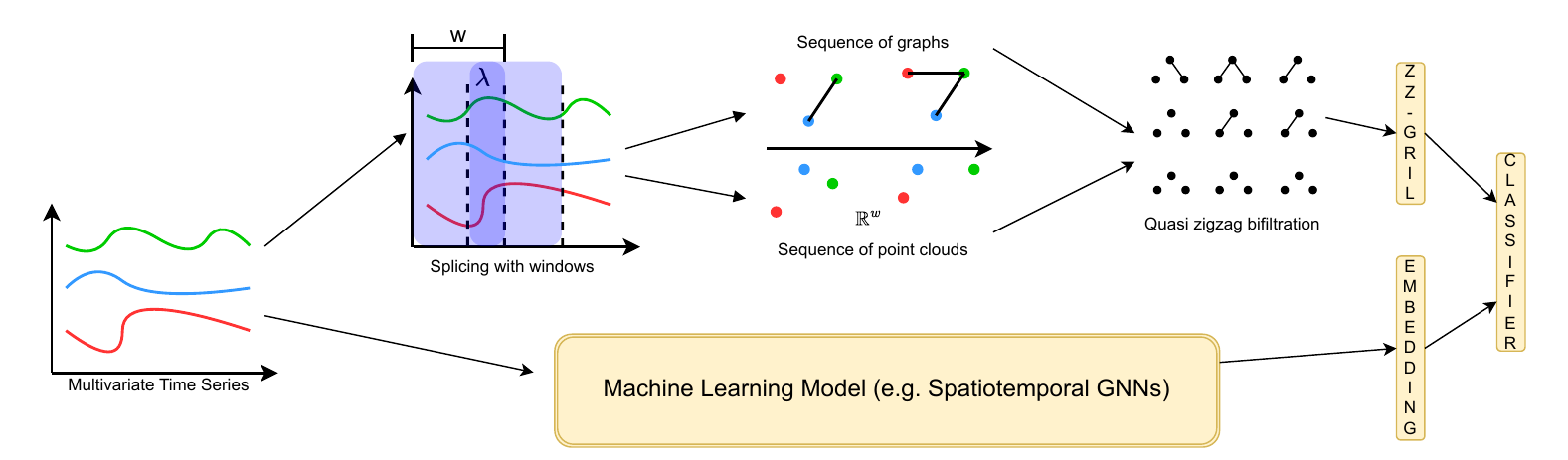}
    \caption{Experimental setup.}
    \label{fig:pipeline}
\end{figure}


\begin{table*}
    \centering
    \resizebox{\textwidth}{!}{
    \begin{tabular}{cccccccccc}
    \toprule
         \textbf{Dataset/Methods}&  ED-1NN&  DTW-1NN-I&  DTW-1NN-D&  MLSTM-FCN&  ShapeNet&  WEASEL+MUSE&   OS-CNN&MOS-CNN &\zzgril\\
         \midrule
         FingerMovements&  0.550&  0.520&  0.530&  0.580&  \textcolor{gray}{0.589}&  0.490&   0.568&0.568 &\textbf{0.590}\\
 Heartbeat& 0.620 & 0.659& 0.717& 0.663& \textbf{0.756}& \textcolor{gray}{0.727}& 0.489& 0.604 &0.721\\
 MotorImagery& 0.510&0.390 &0.500 &0.510 &\textbf{0.610} &0.500 &0.535 &0.515  &\textcolor{gray}{0.580}\\
 NATOPS& 0.860 & 0.850& 0.883& 0.889& 0.883& 0.870& \textbf{0.968}& \textcolor{gray}{0.951} &0.850\\
 SelfRegulationSCP2& 0.483 & 0.533& \textcolor{gray}{0.539}& 0.472& \textbf{0.578}& 0.460& 0.532& 0.510 &0.522\\
 \bottomrule
    \end{tabular}
    }
    \caption{Acccuracy comparison of \zzgril{} with some existing methods on UEA Multivariate Time Series Classification Datasets. Bold entries denote the best model and gray the second best. }
    \label{tab:mvts_results2}
    \vspace{-0.4cm}
\end{table*}

\begin{table}
    \centering
    \begin{tabular}{ccccc}
    \toprule
         \textbf{Dataset/Methods}&  TapNet &TapNet+\zzgril& TodyNet &TodyNet+\zzgril\\
         \midrule
         FingerMovements&  0.530 &0.630&  0.570&\textbf{0.660}\\
 Heartbeat& 0.751 &0.751&\textbf{ 0.756}&\textbf{0.756}\\
 MotorImagery&0.580&0.600&0.640 & \textbf{0.660}\\
 NATOPS& 0.927&0.922& \textbf{0.972}&0.961\\
 SelfRegulationSCP2& 0.538&0.544& 0.550&\textbf{0.600}\\
 \bottomrule
    \end{tabular}
    \caption{Test acccuracy of augmenting \zzgril{} to TapNet and TodyNet on UEA Multivariate Time Series Classification Datasets.}
    \label{tab:mvts_results}
\end{table}

Given a sample of multivariate time-series data with $m$ time-series, we splice each time-series into a \emph{sequence of time-series}, each of length $w$. This splicing is done by a moving window of width $w$, where consecutive windows have an overlap of $\lambda$. We calculate the Pearson correlation coefficient~\cite{pcc} between time-series in each window. We construct a graph for each window, where each time-series (of length $w$) is a node, and edges with the Pearson correlation values as weights connect them. Thus, we get a complete graph with $m$ nodes. We select the top $k$ percentile of these edges to build the final graph in each window. This way we obtain a sequence of graphs. The topological information of this sequence of graphs encodes the evolution of correlation between time series. Alternately, we can also track the evolving time-series by converting it into a sequence of point clouds. From the sequence of time-series described above, we consider each time-series of length $w$ as a point in $\RR^w$. Thus, if we have $m$ time series, we have $m$ points in $\RR^w$ in each window. Thus, we get a sequence of point clouds in $\RR^w$. This sequence of point clouds tracks the evolution of each time-series, and a Vietoris-Rips like construction on this point cloud tracks the evolution of the interaction between time-series. Refer to Figure~\ref{fig:pipeline} for an illustration. 


\subsection{UEA Multivariate Time Series Classification}
UEA Multivariate Time Series Classification (MTSC)~\cite{ueamvts} archive comprises of real-world multivariate time series data and is a widely recognized benchmark in time series analysis. The UEA MTSC collection encompasses a diverse range of application domains such as healthcare (ECG or EEG data), motion recognition (recorded using wearable sensors). See Table~\ref{tab:app:dataset_info} in Appendix~\ref{app:exp}. 

The datasets are preprocessed and split into training and testing sets. 
For this set of experiments, we convert the multivariate time series into a sequence of point clouds and compute \zzgril{}. We choose 5 datasets which have at least 7 multivariate time series. This ensures that each point cloud, in the sequence of point clouds, has at least 7 points, giving meaningful topological information. Refer to Figure~\ref{fig:zz_gril_vis} in Appendix~\ref{app:exp} for a visualization of \zzgril{} on FingerMovements dataset.

In Table~\ref{tab:mvts_results}, we compare the performance of augmenting \zzgril{} to two specifically tailored machine learning models for multivariate time series classification on UEA MTSC datasets, TapNet~\cite{tapnet} and TodyNet~\cite{todynet}. We can see from the table that \zzgril{} is adding meaningful topological information on most datasets which is seen as an improvement in performance. We picked these two models for augmenting \zzgril{} because of two reasons: (i) availability and ease-of-use of codebase, (ii) these models already have good performance on UEA MTSC Datasets and we wanted to test the additional value of the topological information that \zzgril{} adds. In order to do a comprehensive evaluation of \zzgril{} framework, we compare \zzgril{} in isolation with other methods~\cite{li2021,tang2022,karim2019,schafer2017} on UEA MTSC task. We report the results in Table~\ref{tab:mvts_results2}. We see that \zzgril{} has comparable performance. Here, we would like to remind the reader that \zzgril{} is a general framework which does not need to be tailored for these datasets in particular. 
Further, to show that we need both the spatial and the temporal information to be captured together, we compare \zzgril{} with standard \emph{zigzag persistence} at various spatial scales. We report the results in Table~\ref{tab:qzz_vs_zz}. We can see that \zzgril{} performs better than standard zigzag persistence even across different scales.


\begin{table}
    \centering
    \begin{tabular}{ccccc}
    \toprule
         \textbf{Dataset}&  Zigzag (scale = 0.5)&  Zigzag (scale = 0.7)& Zizag (full graph) &\zzgril\\
         \midrule
         FingerMovements&  0.490&  0.490&  0.530&\textbf{0.590}\\
         Heartbeat&  - &  0.682&  0.687&\textbf{0.721}\\
         MotorImagery&  0.510&  0.520&  0.540&\textbf{0.580}\\
         NATOPS&  -&  0.805&  0.816&\textbf{0.850}\\
 SelfRegulationSCP2& 0.483& 0.489& 0.500&\textbf{0.522}\\
 \bottomrule
    \end{tabular}
    \caption{Accuracy of 1-parameter zigzag persistence at different scales compared with \zzgril{}. We use a Random Forest Classifier for classifying the topological signatures. The empty entries denote that the model does not train on the topological features because the features are not distinct enough at that scale.}
    \label{tab:qzz_vs_zz}
\end{table}

We conduct two ablation studies. First, we find that converting time series into either, a sequence of point clouds or a sequence of graphs, is viable for extracting topological information as observed in Table~\ref{tab:app_exp_res}). Second, to highlight the value of dynamic topological information, we show that \zzgril{} outperforms a ``Snapshot PH'' baseline, which vectorizes individual non-zigzag persistence modules. As shown in Table~\ref{tab:app:snapshot_ph_vs_zz}, \zzgril{}'s superior performance confirms the importance of the features captured by zigzag persistence.

\begin{table}[htbp]
    \centering
    \begin{tabular}{ccc}
    \toprule
        \textbf{Dataset} & \textbf{TodyNet + \zzgril{} point clouds} & \textbf{TodyNet + \zzgril{} graphs}\\
        \midrule
         FingerMovements& 0.660 & 0.680\\
         NATOPS &0.961 & 0.945 \\
         SelfRegulationSCP2 &0.600 & 0.594\\
         \bottomrule
    \end{tabular}
    \caption{Comparison of converting multivariate time series as a sequence of point clouds versus converting it as a sequence of graphs to extract the topological information.}
    \label{tab:app_exp_res}
\end{table}

\begin{table}[htbp]
    \centering
    \begin{tabular}{ccc}
    \toprule
        \textbf{Dataset} & \textbf{Snapshot PH}& \textbf{\zzgril}\\
        \midrule
         FingerMovements& 0.480& 0.590\\
 Heartbeat& 0.639&0.721\\
 MotorImagery& 0.560&0.580\\
         NATOPS &0.777& 0.850\\
         SelfRegulationSCP2 &0.466& 0.522\\
         \bottomrule
    \end{tabular}
    \caption{Comparison of vectorizing each non-zigzag filtration individually (Snapshot PH) versus \zzgril{}.}
    \label{tab:app:snapshot_ph_vs_zz}
\end{table}

 We report the computation times in Table~\ref{tab:app:comp_times} which indicate that \zzgril{} is, indeed, practical to use.

 \begin{table}[h!]
    \centering
    \begin{tabular}{ccc}
    \toprule
        \textbf{Dataset} & \textbf{Time (\zzgril{} point clouds)} & \textbf{Time (\zzgril{} graphs)}  \\
        \midrule
        FingerMovements &569s &138s \\
        NATOPS & 413s & 68s \\
        SelfRegulationSCP2 & 153s  & 27s \\
        \bottomrule
    \end{tabular}
    \caption{\small Computation times for \zzgril{}. The values represent the computation times for both training and testing splits. All the experiments were performed on Intel(R) Xeon(R) Gold 6248R CPU and NVIDIA Quattro RTX 6000 GPU.}
    \label{tab:app:comp_times}
\end{table}

\subsection{Sleep Stage Classification}
We use ISRUC-S3 dataset, which is a part of the ISRUC (Iberian Studies and Research on Sleep) Sleep Dataset~\cite{isruc}. 
ISRUC-S3 contains PSG recordings from 10 subjects. Each recording includes multiple physiological signals, such as: EEG, ECG. The dataset is annotated with sleep stage labels for each epoch (30 second window). There are 5 sleep stage labels: Wake (W), N1, N2, N3 (non-REM stages) and REM. We use  STDP-GCN~\cite{stdpgcn} as the machine learning model to augment. In this experiment, we convert the time series into sequence of graphs and compare the performance
in Table~\ref{tab:sleep_classi_results}. We can see that augmenting \zzgril{} increases both the accuracy and the overall F-1 score.

\begin{table}
\vspace{-0.4cm}
    \centering
    \resizebox{\columnwidth}{!}{
    \begin{tabular}{cccccccc}
    \toprule
         \textbf{Methods}&  \textbf{Accuracy}&  \textbf{F-1 Overall}&  \textbf{F-1 Wake}&  \textbf{F-1 N-1}&  \textbf{F-1 N2}&  \textbf{F-1 N3}& \textbf{F-1 REM}\\
         \midrule
         SVM~\cite{alickovic2018}&  73.3&  72.1&  86.8&  52.3&  69.9&  78.6& 73.1\\
         RF~\cite{memar2017}&  72.9&  70.8&  85.8&  47.3&  70.4&  80.9& 69.9\\
         MLP+LSTM~\cite{dong2017}&  77.9&  75.8&  86.0&  46.9&  76.0&  87.5& 82.8\\
         CNN+BiLSTM~\cite{supratak2017}&  78.8&  77.9&  88.7&  60.2&  74.6&  85.8& 80.2\\
         CNN~\cite{chambon2018}&  78.1&  76.8&  87.0&  55.0&  76.0&  85.1& 80.9\\
         ARNN+RNN~\cite{phan2019}&  78.9&  76.3&  83.6&  43.9&  79.3&  87.9& 86.7\\
 STGCN~\cite{jia2020}& 79.9& 78.7& 87.8& 57.4& 77.6& 86.4&84.1\\
 MSTGCN~\cite{jia2021}& 82.1& 80.8& \textbf{89.4}& 59.6& 80.6&\textbf{89.0}&85.6\\
 STDP-GCN~\cite{stdpgcn}& 82.6& 81.0& 83.5& \textbf{62.9}& 83.1& 86.0&90.6\\
 STDP-GCN + \zzgril{}& \textbf{83.8}& \textbf{81.1}& 88.6& 58.1&\textbf{85.4} &82.7 &\textbf{90.9}\\
 \bottomrule
    \end{tabular}}
    \caption{ Accuracy and F-1 scores of augmenting \zzgril{} to STDP-GCN and testing on ISRUC-S3 sleep classification dataset. We report the accuracy and F-1 scores for each model.}
    \label{tab:sleep_classi_results}
\end{table}

We would like to clarify that the aim, for both sets of experiments, is primarily to show that an increase in accuracy upon augmentation signifies that \zzgril{} captures meaningful topological information which can be used to improve the existing models.


\section{Conclusion}
In this paper, we proposed QZPH as a framework to capture both static and dynamic topological features in time-varying data. We proposed \zzgril{}, a stable and computationally efficient topological invariant to address the challenges of integrating MPH and ZPH. Through applications in various domains, including sleep-stage detection, we showed that augmenting machine learning models with \zzgril{} improves the performance. These results highlight the potential of integrating topological information to address complex challenges while analyzing time-evolving data. 


\section{Acknowledgement}
This work is partially supported by NSF grants DMS-2301360 and CCF-2437030. We acknowledge the discussion with Michael Lesnick who pointed out the theory
about initial/terminal functors in the context of limits and colimits.

\newpage
\bibliography{ref}
\bibliographystyle{icml2025}


\newpage
\appendix

\section{Formal definitions and proofs for stability of \zzgril{}}
\label{app:proofs}



There is a notion of proximity on the space of zigzag modules in terms of the \emph{interleaving distance}. In~\cite{zz_stability}, the authors define interleaving distance between two zigzag modules by including them into $\RR^{\mathsf{op}} \times \RR$-indexed modules.

The interleaving distance on $\RR^n$-indexed persistence modules is known and well-defined. We briefly recall the definition here. We refer the readers to~\cite{zz_stability} for additional details.

\begin{definition}[$u$-shift functor]
    The $u$-shift functor $(-)_u \colon \mathbf{vec}^{\RR^n} \to \mathbf{vec}^{\RR^n}$, for $u \in \RR^n$, is defined as follows:
    \begin{enumerate}
        \item For $M \in \mathbf{vec}^{\RR^n}$, $M_u$ is defined as $M_u(x) = M(x+u)$ for all $x \in \RR^n$ and $M_u(x_1 \leq x_2) = M(x_1 + u \leq x_2 + u)$ for all $x_1 \leq x_2 \in \RR^n$, 
        \item Let $M,N \in \mathbf{vec}^{\RR^n}$. Let $F \colon M \to N$ be a morphism. Then, the corresponding morphism $F_u \colon M_u \to N_u$ is defnied as $F_u(x) = F(x+u) \colon M_u(x) \to N_u(x)$ for all $x \in \RR^n$.
    \end{enumerate}
\end{definition}

\begin{definition}[$\epsilon$-interleaving]
    Let $M, N \in \mathbf{vec}^{\RR^n}$. Let $\epsilon \in [0,\infty)$ be given. We will denote $(-)_{\bm{\epsilon}}$ to be the shift functor corresponding to the vector $\bm{\epsilon} = \epsilon(1,1,\hdots,1)$. We say $M$ and $N$ are $\epsilon$-interleaved if there are natural transformations $F \colon M \to N_{\bm{\epsilon}}$ and $G \colon N \to M_{\bm{\epsilon}}$ such that 
    \begin{enumerate}
        \item $G_{\bm{\epsilon}} \circ F = \varphi_M^{2\bm{\epsilon}} $,
        \item $F_{\bm{\epsilon}} \circ G = \varphi_N^{2\bm{\epsilon}}$,
    \end{enumerate}
    where $\varphi_M^u \colon M \to M_u$ is the natural transformation whose restriction to each $M(x)$ is the linear map $M(x \leq x+ u)$ for all $x \in \RR^n$.
\end{definition}

\begin{definition}[Interleaving distance]
    Let $M, N \in \mathbf{vec}^{\RR^n}$. The interleaving distance $d_\mathcal{I}(M,N)$ between $M$ and $N$ is defined as
    \begin{equation*}
        d_\mathcal{I}(M,N) \coloneqq \inf \{\epsilon \geq 0 \colon M \text{ and } N \text{ are } \epsilon\text{-interleaved} \}
    \end{equation*}
    and $d_\mathcal{I}(M,N) = \infty$ if there exists no interleaving.
    \label{def:interleaving}
\end{definition}

We need the following two definitions to define interleaving distance between two zigzag modules. 

\begin{definition}[Left Kan Extension]
    Let $P$ and $Q$ be two posets. Let $F \colon P \to Q$ be a functor. Let $P[F \leq q]$ denote the set $P[F \leq q] \coloneqq \{p \in P \colon F(p) \leq q \}$. Given a persistence module $M \colon P \to \mathbf{vec}$, the \emph{left Kan extension} of $M$ along $F$ is a functor $\text{Lan}_F(M) \colon Q \to \mathbf{vec}$ given by 
    \begin{equation*}
        \text{Lan}_F(M)(q) \coloneqq \varinjlim M|_{P[F \leq q]},
    \end{equation*}
    along with internal morphisms given by the universality of colimits.
\end{definition}

\begin{definition}[Block Extension Functor~\cite{zz_stability}]
    Let $\uu \subset \ropr$ denote the poset $\mathbb{U} \coloneqq \{(a,b) \colon a \leq b\}$. Let $\mathfrak{i} \colon \zz \to \RR^{\mathsf{op}} \times \RR$ denote the inclusion. Let $(-)|_\uu \colon \mathbf{vec}^{\ropr} \to \mathbf{vec}^{\uu}$ denote the restriction. Then, the \emph{block extension functor} $E \colon \mathbf{vec}^{\zz} \to \mathbf{vec}^\uu$ is defined as 
    \begin{equation*}
        E \coloneqq (-)|_\uu \circ \text{Lan}_\mathfrak{i}(\circ).
    \end{equation*}
\end{definition}


\begin{definition}[Interleaving distance on zigzag modules]
    Let $M,N$ be two zigzag modules. Then
    \begin{equation*}
        d_\mathcal{I}(M,N) \coloneqq d_\mathcal{I}(E(M), E(N)).
    \end{equation*}
\end{definition}

We extend this definition of interleaving distance to \qzzmod s.  Let $\uuu \subset \ropr \times \RR$ be the analog of $\uu$, i.e, $\uuu \coloneqq \{ (a,b,c) \colon a\leq b \}$. Let $\Tilde{\mathfrak{i}} \colon \zz \times \z \to \ropr \times \RR$ denote the inclusion. Then, define $\Tilde{E} \coloneqq (-)|_\uuu \circ \text{Lan}_{\Tilde{\mathfrak{i}}}(\circ)$, analogous to the block extension functor.

\begin{definition}
    Let $M, N$ be two \qzzmod s. Then, the interleaving distance between $M$ and $N$ is given by 
    \begin{equation*}
        d_\mathcal{I}(M,N) \coloneqq d_\mathcal{I}(\Tilde{E}(M), \Tilde{E}(N)).
    \end{equation*}
\end{definition}

\begin{lemma}\label{alem:interval}
    $\Tilde{E}$ sends interval modules on $\zz \times \z$ to block interval modules.
\end{lemma}
This is analogous to Lemma 4.1 in~\cite{zz_stability}.


\begin{lemma}\label{alem:decomposition}
    Let $M$ be a \qzzmod. If $M \cong \bigoplus_{k\in K} I_k$, then $\Tilde{E}(M) \cong \bigoplus_{k \in K} \Tilde{E}(I_k)$, where $K$ is an indexing set.
\end{lemma}

\begin{proof}
    The functor $\text{Lan}_{\Tilde{\mathfrak{i}}}$ preserves direct sums~\cite{Saunders_Maclane_Cat_Theory}, as it is left-adjoint to the canonical restriction functor. The restriction functor also preserves direct sums. These two facts combined with Lemma~\ref{alem:interval} give the result.
\end{proof}

\begin{definition}
\label{adef:erosion_dist}
    Let \textbf{I}$(\zz \times \z)$ denote the collection of all subposets in $\zz \times \z$ such that their corresponding subposets in $\z^2$ are intervals. Let $M$ and $N$ be two \qzzmod s. The erosion distance is defined as:
    \begin{equation*}
    \begin{split}
        d_\mathcal{E}(M,N) \coloneqq \inf \limits_{\epsilon \geq 0} 
        \{ & \forall I \in \textbf{I}(\zz \times \z), \\
        & \RKG^M(I) \geq \RKG^N(I^\epsilon) \text{ and } \\
        & \RKG^N(I) \geq \RKG^M(I^\epsilon)  \}
    \end{split}
    \end{equation*}
\end{definition}

\begin{definition}
    Let $\mathcal{L}$ denote the collection of all worms in $\zz \times \z$. Let $M$ and $N$ be \qzzmod s. The erosion distance can be defined as:
    \begin{equation*}
    \begin{split}
        d_\mathcal{E}^\mathcal{L}(M,N) \coloneqq \inf \limits_{\epsilon \geq 0} 
        \{ & \forall \boxed{\vp}^2_\delta \in \textbf{I}(\zz \times \z), \\
        & \RKG^M\left (\boxed{\vp}^2_\delta\right) \geq \RKG^N\left(\boxed{\vp}^2_{\delta+\epsilon}\right ) \text{ and } \\
        & \RKG^N\left (\boxed{\vp}^2_\delta\right ) \geq \RKG^M\left (\boxed{\vp}^2_{\delta+\epsilon}\right)  \}.
    \end{split}
    \end{equation*}
\end{definition}


\begin{propositionof}{\ref{prop:interleaving}} \label{aprop:interleaving}
    Given two \qzzmod s $M$ and $N$, $d_\mathcal{E}^{\mathcal{L}}(M,N) \leq d_\mathcal{I}(M,N)$ where $d_\mathcal{I}$ denotes the interleaving distance between $M$ and $N$.
\end{propositionof}

\begin{proof}
First, it is obvious that $d_\mathcal{E}^{\mathcal{L}}(M,N) \leq d_\mathcal{E}(M,N)$. Now, $d_\mathcal{I}(M,N) \coloneqq d_\mathcal{I}(\Tilde{E}(M), \Tilde{E}(N))$. Let $I$ be a subposet in $\zz \times \z$ such that its corresponding subposet in $\z^2$ is an interval.  By Lemma~\ref{alem:interval} and Lemma~\ref{alem:decomposition}, and the fact that generalized rank over a given subposet counts the number of intervals 
that contain the given subposet, we get $\RKG^M(I) = \RKG^{\Tilde{E}(M)}(\Tilde{\mathfrak{i}}(I) \cap \uuu)$. This gives us $d_\mathcal{E}(M,N) = d_\mathcal{E}(\Tilde{E}(M),\Tilde{E}(N))$. In~\cite{GenRankKim21}, the authors show that $d_\mathcal{E}(V,W) \leq d_\mathcal{I}(V,W)$, where $V,W \colon \RR^n \to \textbf{vec}$ are $\RR^n$-indexed persistence modules. Thus, we get $d_\mathcal{E}^{\mathcal{L}}(M,N) \leq d_\mathcal{E}(M,N) = d_\mathcal{E}(\Tilde{E}(M),\Tilde{E}(N)) \leq d_\mathcal{I}(\Tilde{E}(M),\Tilde{E}(N)) = d_\mathcal{I}(M,N)$. 

\end{proof}



\begin{theoremof}{\ref{thm:stability_zzgril}}\label{aprop:stability_zzgril}
    Let $M$ and $N$ be two \qzzmod s. Let $\lambda^M$ and $\lambda^N$ denote the \zzgril{} functions of $M$ and $N$  respectively. Then, 
    \begin{equation*}
        ||\lambda^M - \lambda^N||_\infty = d_\mathcal{E}^{\mathcal{L}}(M,N) \leq d_\mathcal{I}(M,N).
    \end{equation*}
\end{theoremof}

\begin{proof}
    We show that $||\lambda^M - \lambda^N||_\infty = d_\mathcal{E}^\mathcal{L}(M,N)$. 
    
    To see $||\lambda^M - \lambda^N||_\infty \leq d_\mathcal{E}^\mathcal{L}(M,N)$, fix $\vp, k$, and let $\lambda^M(\vp,k) = \delta_M$ and $\lambda^N(\vp,k) = \delta_N$. WLOG, assume $\delta_M \geq \delta_N$. Let $d_\mathcal{E}^\mathcal{L}(M,N) = \epsilon$. Therefore, by definition of $d_\mathcal{E}^\mathcal{L}$, we have $\RKG^M\left (\boxed{\vp}^2_{\delta_M}\right) = k$ and $\RKG^M\left (\boxed{\vp}^2_{\delta_N + \epsilon}\right) \leq \RKG^N\left (\boxed{\vp}^2_{\delta_N}\right) = k$. Thus, by the definition of \zzgril{}, we get $\delta_N + \epsilon \geq \delta_M$, i.e., $\delta_M - \delta_N \leq \epsilon = d^\mathcal{L}_\mathcal{E}$.

    To see $||\lambda^M - \lambda^N||_\infty \geq d_\mathcal{E}^\mathcal{L}(M,N)$, fix $\vp, k$, and let $|\lambda^M(\vp,k) - \lambda^N(\vp,k)| = \epsilon$. Let $\boxed{\vp}^2_\delta$ be a worm. Let $k = \RKG^N\left (\boxed{\vp}^2_{\delta + \epsilon}\right)$. Then, $\lambda^N(\vp,k) \geq \delta + \epsilon$. Thus, $|\lambda^M(\vp,k) - \lambda^N(\vp,k)| = \epsilon$ and $\lambda^N(\vp,k) \geq \delta + \epsilon$ give $\lambda^M(\vp,k) \geq d$. Therefore, $\RKG^M\left (\boxed{\vp}^2_{\delta}\right) \geq k = \RKG^N\left (\boxed{\vp}^2_{\delta + \epsilon}\right)$. Similarly, we can show $\RKG^N\left (\boxed{\vp}^2_{\delta}\right) \geq  \RKG^M\left (\boxed{\vp}^2_{\delta + \epsilon}\right)$. Thus, by definition of $d_\mathcal{E}^\mathcal{L}$, we get that $||\lambda^M - \lambda^N||_\infty \geq d_\mathcal{E}^\mathcal{L}(M,N)$.
    
\end{proof}

\section{Experimental Details}

\label{app:exp}
Here, we report additional details about our experiments. We begin by including a detailed description of the UEA datasets in Table~\ref{tab:app:dataset_info}. Further, we have a comparison of treating a multivariate time series as a sequence of point clouds versus treating it as a sequence of graphs as the input to the \zzgril{} framework. We report the results of this experiment in Table~\ref{tab:app_exp_res}. We can see that there is no clear winner. On some datasets, \zzgril{} extracts more relevant information from sequences of point clouds while on some datasets, information from sequences of graphs performs better. We have a visualization of \zzgril{} in Figure~\ref{fig:zz_gril_vis}. We can see in the figure that the \zzgril{} topological signature is different for samples of different classes and very similar for samples in the same class.

For all the experiments with sequence of point clouds, we choose a window size that is $\max(5, series\_length//128)$ and overlap $\max(4, 0.7*window\_size)$. For the experiments with sequence of graphs, we choose a window size of $\min(series\_length/5,128)$ and an overlap of $0.7*window\_size$. Then, from the complete graph, we randomly choose a percentage between 65 and 75 of the edges depending on their weights. This is to ensure that the number of edges does not remain the same for all the graphs in the sequence. For all our experiments, we use 36 center points to compute \zzgril{} at. For model parameters, mostly, we use the same parameters as specified by the respective models~\cite{todynet, stdpgcn}. We notice that the training is sensitive to learning rate and we optimize the learning rate between $1e-3$ and $5e-5$ for different datasets. These choice of hyperparameters are based on preliminary experiments, the results of which we report in Table~\ref{tab:app:hyperparam1}, Table~\ref{tab:app:hyperparam2} , Table~\ref{tab:app:hyperparam3} and Table~\ref{tab:app:hyperparam4}. The results reported here are for TodyNet+\zzgril{} model.

\begin{table}[htbp]
    \centering
    \resizebox{\columnwidth}{!}{
    \begin{tabular}{ccccccc}
    \toprule
         \textbf{Dataset}&  \textbf{Type}&  \textbf{No. of series}&  \textbf{Series length}& \textbf{No. of classes} &\textbf{Train size} &\textbf{Test size}\\
         \midrule
 FingerMovements& EEG& 28& 50& 2& 316&100\\
         Heartbeat&  AUDIO&  61&  405&  2& 204&205\\
 MotorImagery& EEG& 64& 3000& 2& 278&100\\
         NATOPS&  HAR&  24&  51&  6& 180&180\\
         SelfRegulationSCP2&  EEG&  7&  1152&  2& 200&180\\
         \bottomrule
    \end{tabular}}
    \caption{Information about UEA Datasets used for experiments}
    \label{tab:app:dataset_info}
\end{table}

\begin{table}[h!]
    \centering
    \begin{tabular}{cccc}
    \toprule
        \textbf{Dataset} & \textbf{36 center points} & \textbf{25 center points} & \textbf{64 center points}\\
        \midrule
         FingerMovements& 0.660 & 0.620 & 0.680 \\
         SelfRegulationSCP2 &0.600 & 0.578 & 0.600 \\
         \bottomrule
    \end{tabular}
    \caption{Selection of number of center points to compute \zzgril{}}
    \label{tab:app:hyperparam1}
\end{table}

\begin{table}[h!]
    \centering
    \resizebox{\columnwidth}{!}{
    \begin{tabular}{cccc}
    \toprule
        \textbf{Dataset} & $\mathbf{max(0.7 * window\_size,5)}$  & $\mathbf{max(0.3 * window\_size,2)}$ & $\mathbf{max(0.5 * window\_size,3)}$\\
        \midrule
         FingerMovements& 0.660 & 0.620 & 0.640\\
         SelfRegulationSCP2 &0.600 & 0.550 & 0.578\\
         \bottomrule
    \end{tabular}}
    \caption{Selection of overlap between time windows for time-series modelled as a sequence of graphs. The $window\_size$ is fixed to be $max(128//series\_len, 5$).}
    \label{tab:app:hyperparam2}
\end{table}

\begin{table}[h!]
    \centering
    \resizebox{\columnwidth}{!}{
    \begin{tabular}{cccc}
    \toprule
        \textbf{Dataset} & $\mathbf{max(series\_len//128,5)}$ & $\mathbf{max(series\_len//64,10)}$ & $\mathbf{max(series\_len//256,3)}$\\
        \midrule
         FingerMovements& 0.660 & 0.640 & 0.580\\
         SelfRegulationSCP2 &0.600 & 0.578 & 0.550\\
         \bottomrule
    \end{tabular}}
    \caption{Selection of window size for time-series modelled as a sequence of graphs. The overlap is fixed to be $max(4, 0.7 * window\_size)$.}
    \label{tab:app:hyperparam3}
\end{table}

\begin{table}[h!]
    \centering
    \resizebox{\columnwidth}{!}{
    \begin{tabular}{cccc}
    \toprule
        \textbf{Dataset} & \textbf{Threshold (45\% - 55\%)}& \textbf{Threshold (65\% - 75\%)}& \textbf{Threshold (85\% - 95 \%)}\\
        \midrule
         FingerMovements& 0.600& 0.660& 0.630\\
         SelfRegulationSCP2 &0.550& 0.600& 0.550\\
         \bottomrule
    \end{tabular}}
    \caption{Selection of thresholding value range for retaining edges in the graph.}
    \label{tab:app:hyperparam4}
\end{table}

\begin{figure}[htbp]
    \centering
    \includegraphics[scale=0.8]{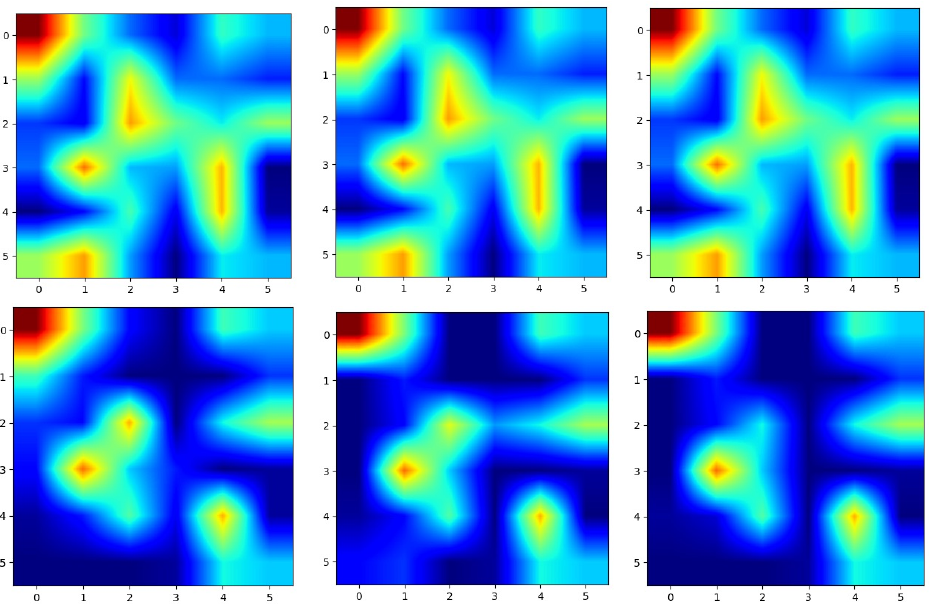}
    \caption{Heatmap of \zzgril{} on the FingerMovements dataset. \zzgril{} is computed at 36 center points, which is plotted as a 6x6 grid. The top row represents three samples with label 1 and the bottom row represents three samples with label 0. We can see the similarity in the \zzgril{} signatures between samples of the same class and clear differences between samples of different classees.}
    \label{fig:zz_gril_vis}
\end{figure}


\end{document}